\documentclass[nohyperref]{article}

\usepackage{microtype}
\usepackage{graphicx}
\usepackage{booktabs} 
\usepackage[]{subcaption}

\usepackage{hyperref}

\usepackage[arxiv]{icml2022}

\usepackage{amsmath}
\usepackage{amssymb}
\usepackage{mathtools}
\usepackage{amsthm}
\usepackage{algorithm}
\usepackage{algorithmic}
\usepackage{caption}
\usepackage{subcaption}
\usepackage{graphicx}
\usepackage{amsmath}
\usepackage{amsmath}
\usepackage{mathtools}
\usepackage{xcolor}
\usepackage{amssymb}
\usepackage{amsthm}
\usepackage{soul}
\usepackage{bbm}
\usepackage[]{natbib}   

\usepackage[capitalize,noabbrev]{cleveref}

\theoremstyle{plain}
\newtheorem{theorem}{Theorem}[section]

\newtheorem{lemma}[theorem]{Lemma}
\newtheorem{corollary}[theorem]{Corollary}
\theoremstyle{definition}
\newtheorem{definition}[theorem]{Definition}

\theoremstyle{remark}

\usepackage{thmtools} 
\usepackage{thm-restate}

\usepackage[textsize=tiny]{todonotes}

\newcommand{\Var}{\operatorname{Var}}

\icmltitlerunning{Improved Policy Evaluation for Randomized Trials of Algorithmic Resource Allocation}

\begin{document}

\twocolumn[
\icmltitle{Improved Policy Evaluation for Randomized Trials of Algorithmic Resource Allocation}



\icmlsetsymbol{equal}{*}
\icmlsetsymbol{intern}{\dag}

\begin{icmlauthorlist}
\icmlauthor{Aditya Mate}{yyy,comp,intern}
\icmlauthor{Bryan Wilder}{zzz}
\icmlauthor{Aparna Taneja}{comp}
\icmlauthor{Milind Tambe}{yyy,comp}
\end{icmlauthorlist}

\icmlaffiliation{yyy}{Harvard University, USA}
\icmlaffiliation{zzz}{Carnegie Mellon University, USA}
\icmlaffiliation{comp}{Google Research, India}

\icmlcorrespondingauthor{Aditya Mate}{aditya\_mate@g.harvard.edu}

\icmlkeywords{Machine Learning, ICML}

\vskip 0.3in
]



\printAffiliationsAndNotice{\icmlinternship} 

\begin{abstract}
We consider the task of evaluating policies of algorithmic resource allocation through randomized controlled trials (RCTs). Such policies are tasked with optimizing the utilization of limited intervention resources, with the goal of maximizing the benefits derived. Evaluation of such allocation policies through RCTs proves difficult, notwithstanding the scale of the trial, because the individuals' outcomes are inextricably interlinked through resource constraints controlling the policy decisions. 
Our key contribution is to present a new estimator leveraging our proposed novel concept, that involves retrospective reshuffling of participants across experimental arms at the end of an RCT. We identify conditions under which such reassignments are permissible and can be leveraged to construct counterfactual trials, whose outcomes can be accurately ascertained, for free. We prove theoretically that such an estimator is more accurate than common estimators based on sample means --- we show that it returns an unbiased estimate and simultaneously reduces variance. We demonstrate the value of our approach through empirical experiments on synthetic, semi-synthetic as well as real case study data and show improved estimation accuracy across the board.
\end{abstract}

\section{INTRODUCTION}

We consider a subclass of randomized controlled trials (RCTs) wherein the goal of the trial is to evaluate the efficacy of an algorithmic resource allocation policy. Such policies recommend an allocation (action), commonly utilizing tools such as reinforcement learning \cite{galstyan2004resource, tesauro2006hybrid}, variations of the multi-armed bandit framework \cite{jamieson2014lil}, network optimization \cite{wilder2021clinical, vaswani2015influence}, etc. As machine learning becomes increasingly widely applied in socially critical settings, such policies have been used to allocate limited resources in a variety of domains, including  campaign optimization \citep{Leskovec2007campaign, eagle2010network}, improving maternal healthcare \citep{biswas2021learn}, screening for hepatocellular carcinoma \citep{lee2019optimal}, monitoring tuberculosis patients \citep{mate2021risk},
etc. Such a policy may rank all individuals within a group, and offer a scarce resource, such as a home visit by a health worker, to the highest-ranked individuals within the group.

RCTs evaluating these resource allocation policies consist of $M$ experimental arms labeled $\{1,\dots, M \}$, with each arm consisting of $N$ unique, randomly assigned individuals. The policy in each experimental arm prescribes an allocation of resources while respecting some resource constraints. This process may be either single-shot (allocations made once) or sequential (new allocation decisions made adaptively over a series of rounds). Each policy, dictating its own resource allocation strategy, is evaluated at the end of the trial by analyzing the outcomes data from its corresponding experimental arm. The group-level decision making of resource allocation policies creates new challenges for their experimental analysis. 
In a standard RCT, the aim is to evaluate the treatment effect -- there is no resource constraint, so all participants receive treatment -- for the average participant \cite{angrist2009mostly}. As the outcomes of all participants are independent, analysts can simply compare outcomes in the treatment versus control groups. However, trials of resource allocation policies aim to evaluate the \textit{group} outcome of a set of individuals to whom the policy is applied (some of whom receive the resource, and some who do not). Even when the number of individuals in the trial is large by the standards of a normal RCT, randomized trials of allocation policies can suffer from high variance, leading to noisy and potentially erroneous estimates of the treatment effect. 


This high variance stems from two sources. First, the outcomes of individuals within a given arm are correlated because allocation policies typically consider all individuals \textit{jointly} in making an allocation decision (e.g., to respect budget constraints on the total number of individuals who may receive an intervention). This interdependence implies that we only see one independent sample of a policy's performance per RCT, instead of many (in contrast with one sample per participant in a standard RCT).
Second, for many individuals the allocation decisions made by different policies being compared may coincide. For instance, no matter the policy applied, many individuals may consistently get screened in, or many may get screened out of receiving a resource. Because these individuals are oblivious to the policy employed and receive identical allocations under different test policies, their outcomes are never truly impacted by the policy employed. Yet their presence generates additional variance in the average outcome of the group due to stochasticity in their outcomes, independent of the policy being evaluated.
\textit{To our knowledge, no prior work has proposed strategies to mitigate these sources of variance in RCTs of resource allocation algorithms, despite the increasing use of ML-based policies in socially critical domains.}

An intuitive first attempt at variance reduction would be to share information across arms. For example, we could average out fluctuations in the outcomes of individuals who do not receive an allocation by pooling together such individuals across all of the test policies. Unfortunately, such naive estimators are biased because the distribution of \textit{which} individuals receive a resource (or not) is different across the trial arms -- indeed, assessing the comparative efficacy of these distributions is exactly the point of the trial. A more sophisticated strategy would be to exploit overlap between the policies via inverse propensity weighting estimators which reweight participants in one arm to match the actions a different policy would have taken \cite{kuss2016propensity, austin2015moving}. However, such estimators are themselves subject to notoriously high variance, particularly when the overlap between policies is poor \cite{zhou2020propensity}. Moreover, as we discuss in Section~\ref{sec:methodology-propensity-scores}, propensity scores are impossible to calculate for sequential intervention settings, where unseen states prevent us from evaluating individual-level propensities for a different policy.

Our key contribution fixing these issues is a novel estimator for the treatment effect of resource allocation policies which exploits overlap between arms in a principled manner, is guaranteed to reduce variance, and is applicable to either single-stage or sequential settings. The main idea is to find a subset of individuals with the following property: if we ran a hypothetical trial where the arm assignments of these individuals were swapped, the allocations made for all individuals remain unchanged. Our estimator identifies such sets of individuals with this property and averages the outcomes of all corresponding hypothetical trials (which are observable because no allocation decisions were altered from the original trial). We make the following contributions: (1) we propose this novel estimator and prove that it produces an unbiased estimate of the treatment effect with a guaranteed reduction in variance compared to the standard estimator, implying that it has strictly smaller average error; (2) we show how this estimator can be implemented in a computationally efficient manner for a class of policies encompassing those most commonly used in real-world allocation decisions; (3) we conduct experiments on three domains leveraging synthetic, semi-synthetic, and real-world data. We show the application of our techniques to real-world case study data illustrating its usefulness in solically critical domains. Across the board, our estimator substantially cuts error (by up to $\sim70\%$) compared to available estimates. 

\section{PROBLEM FORMULATION}

\paragraph{RCT Setup:}

We consider a set of $N$ individuals. Each has a feature vector $\textbf{x} \coloneqq [\textbf{x}_o, \textbf{x}_u]$, where $\textbf{x}_o \in \mathcal{R}^o$ denotes the individual's observable features (such as demographics) while $\textbf{x}_u \in \mathcal{R}^u$ denotes unobservable characteristics which nevertheless influence outcomes. 
A resource allocation policy $\pi$ jointly considers all $N$ individuals, their joint matrix of observable features $\textbf{X}_o \in \mathcal{R}^{N \times o}$,
and prescribes an allocation (action) vector $ \textbf{a} \coloneqq \pi(\textbf{X}_o) \in \mathcal{A}^N$ where $\mathcal{A}$ denotes the space of possible actions for each individual. The allocations $\textbf{a}$ must respect resource constraints specific to the domain. On example is a budget constraint $||\textbf{a}|| \le B$ capping the total cost of allocated actions. Each individual then stochastically yields an outcome state $s\in \mathcal{S}$, according to an unknown function $P^*(\textbf{x},a,s)$ denoting the probability of the individual receiving outcome $s$.
We use $r(s,a)$ to denote the reward accrued by the policymaker from this outcome. We remark that this notation is equivalent to the standard potential outcomes framework \cite{rubin2005causal} where $r(s, a)$ is the realization of the individual's potential outcome given action $a$ and $P$ governs the joint distribution between $\textbf{x}$ and the potential outcomes. We adopt the present notation for easy generalization to the sequential setting.

The goal of an RCT is to compare $M$ such resource allocation policies using $M$ randomly constructed experimental arms $C_1, \dots C_M$ where $C_i$ denotes the set of individuals assigned to the $i^{th}$ arm. We use $\mathfrak{C} \coloneqq \{C_1, \dots C_M\}$ to denote this particular assignment of individuals to experimental arms, chosen uniformly at random from $\mathcal{C}$, which denotes the universal set of all possible assignments. At the end of an RCT, the analyst can compare the performance of the $M$ allocation policies by comparing the sum total of rewards accrued by participants within each arm, given as  $\textrm{Eval}(\pi_m) \coloneqq \sum_{i \in C_m}r\big(s(i),a(i)\big)$. 

\textbf{Sequential RCTs:} We also consider RCTs involving a multi-stage resource allocation setup. Such RCTs run for a total of $T$ rounds (the single-stage setting above corresponds to $T=1$), where allocations $\textbf{a}_t \in \mathcal{A}^N$ must be made at each timestep $t \in [T]$ subject to potentially time-dependent resource constraints. At each time $t$, each individual has a state $s$ belonging to state space $\mathcal{S}$. We use $\textbf{s}_t$ to denote the $N$ individuals' states at timestep $t$. The policy $\pi$ may consider the entire history of previous states and actions to generate a new action allocation as 
$\textbf{a}_t \coloneqq \pi(\textbf{X}_o, \textbf{s}_{0: t-1}, \textbf{a}_{1:t-1})$. Similarly, the individual transitions to a new state $s_{t+1}$ according to the probability function $P^*(\textbf{x}, s_{0:t}, a_{1:t+1}, s_{t+1})$.
The state and action trajectories of all $N$ participants are recorded for each of the $M$ experimental arms as matrices: $S_1, \dots, S_M \in \mathcal{S}^{N \times T+1}$ and $A_1, \dots, A_M \in \mathcal{A}^{N \times T}$, which facilitates similar computation of $\textrm{Eval}(\pi_m) \coloneqq \sum_{i \in C_m}r(S[i], A[i])$ 
as defined for the single-shot setting.

\textbf{Index-based policy:}
We define a specific class of policies -- called `index-based policies' -- for which a computationally efficient estimator can be derived. Index-based policies rank individuals according to some scoring rule that determines a prioritization among individuals for allocation of a resource. Such policies encompass most relevant resource allocation policies commonly employed in practice due to their transparency and ease of implementation \cite{ustun2019learning}. Note that this class includes seemingly unlikely candidates that allocate resources to individuals cyclically in a set order or benchmarks such as `control' groups (details in Appendix~\ref{app:index-based-policies}). Formally, index-based policies are defined for a binary action space $\mathcal{A} \coloneqq \{0,1\}$ with a budget constraint allowing at most $B$ individuals to receive the action $a=1$. The policy computes a time-dependent index $\Upsilon(\textbf{x}_o, s_{0:t-1}, a_{1:t-1})$ for each individual at each timestep $t$, based solely on the individual's  observable features $\textbf{x}_o$, trajectory of states $s$ and history of actions $a$ received.
The policy $\pi$ allocates action $a=1$ to the top $B$ individuals with the largest values of index $\Upsilon$. At any given time $t$, we assume the value of $\Upsilon$ is unique for each individual, i.e., the policy induces a total ordering. The key feature of an index-based policy is that $\Upsilon$ is computed independently for each individual based only on their own features and history. 



\textbf{Problem Statement:} 
$\textrm{Eval}(\pi_m)$ 
provides a single random sample with which to estimate the \textit{expected} performance 
of allocation policy $\pi_m$, combining the randomness in the assignments $\mathfrak{C} \in \mathcal{C}$ and the randomness in outcomes within each $\mathfrak{C}$, engendered by stochasticity in state transitions. Let $\textrm{Eval}^*(\pi_m)$ 
be the expected value of the performance of $\pi_m$: 
\begin{align}
    \label{eq:eval-star-definition}
    \textrm{Eval}^*(\pi_m) & \coloneqq \mathbb{E}_{S_m \sim P^*}\mathbb{E}_{\mathfrak{C} \sim \mathcal{C}}\Big[\textrm{Eval}(\pi_m)\Big]
\end{align}
We assume data available from only a single run of an RCT. Our goal is to build an estimator that estimates $\textrm{Eval}^*(\pi_m)$ 
accurately from just the single RCT instance. We remark that our results make no distributional assumptions about the set of individuals in the trial, e.g., we do not require them to be IID from some distribution. Rather, we treat the observed individuals as fixed (nonrandom) and propose techniques that use only the randomness in the assignment process of the RCT. The only assumption required is that individuals' state transitions are independent, akin to the standard stable unit treatment values assumption (SUTVA) in causal inference \cite{hernan2010causal}. This is aligned with the growing emphasis on design-based inference in causal inference (c.f. \cite{hudgens2008toward,mukerjee2018using,abadie2022should}), which allows us to formulate methods which require minimal modeling assumptions.   


\section{RELATED WORK}

\paragraph{Off-Policy Evaluation:}
The most closely related previous work is the off-policy evaluation (OPE) literature, where the goal is to use samples collected under some baseline policy to inform the evaluation of a new policy \cite{sutton2008convergent,jiang2016doubly,wang2017optimal}. OPE makes frequent use of inverse propensity weighting estimators \cite{Liu2010,li2015toward}; we discuss strategies for constructing such estimators as well as their disadvantages in Section 4.1. To date, the OPE literature has largely focused on individual-level decisions, as opposed to the group-level resource allocation we study. One exception is slate-level OPE, where the policy recommends a ranked list of items to a user \cite{swaminathan2017off}. Slate OPE is most similar to our single-step case, while we develop methods that extend to the multi-step setting. Additionally, slate OPE is motivated by unobservable individual-level rewards, while we assume that individual rewards are observable and the challenge for the single-step setting is that policies may be deterministic (preventing us from using their methods).

\textbf{Individual treatment rules: } Recent work in statistics has studied experimental design and analysis for individual treatment rules (ITR), which are similar to our class of index-based policies \cite{imai2021experimental,athey2021policy}. The crucial difference is that ITRs make decisions independently for each individual, while in our setting the policy considers the group of individuals jointly, which is required for exact enforcement of constraints such as budgets. Our techniques are motivated by the need to reduce variance when policies can only be evaluated at group level. 

\textbf{Cluster-randomization and interference: } Our setting is related to a family of RCTs known as cluster-randomized trials (see \cite{hayes2017cluster} for an overview). In such trials, treatment is assigned at a group level (e.g., assignment of classrooms within a school instead of students), just as groups of individuals are assigned to policies in our setting. However, cluster-randomized trials are motivated by the potential for spillover effects, where the outcomes of one unit can influence others. By contrast, in our setting the outcomes of individuals are still independent conditioned on the actions of the policy. Accordingly, there is no need to account for potential correlations as in the interference literature; instead, we leverage this structured independence to develop lower-variance estimators. 



\section{METHODOLOGY}
\label{sec:methodology}
Our goal is to leverage the overlap in decisions of multiple resource allocation policies to improve our estimate of the reward from deploying each. We start by developing an estimator using inverse propensity weighting -- a popular approach typically adopted for such a task -- and show how it can be naturally applied to the single-stage setting. However, this natural estimator suffers from two challenges. First, inverse propensity estimators can suffer from notoriously high variance, a phenomenon that we empirically confirm in Section~\ref{sec:experiments}. Second, the approach breaks down entirely in the multi-stage setting, where (as detailed below) computation of propensities is impossible due to missing data. We resolve these challenges by developing a more stable ``assignment permutation" estimator, which applies to both settings and is guaranteed to reduce estimation error.  


\subsection{Propensity Scores Approach}
\label{sec:methodology-propensity-scores}
In typical off-policy evaluation settings, inverse propensity weighting (IPW) methods reweight samples according to the probability that observed actions would be taken by a given policy. These methods are not immediately applicable to our problem because we do not assume that policies are randomized -- indeed, explicit randomization is rare in policies deployed by real-world governments, health systems, etc. When policies are deterministic, the probability that they would yield an alternate action is precisely zero, leaving standard propensity estimators undefined. 

We show how to circumvent this issue in the single-step setting by leveraging an alternate source of randomness: the assignment in the trial itself. Calculated over the randomness in the assignment, each individual has some probability of being assigned a given action, denoted as $Pr_{\mathfrak{C}, \pi}[a(i) = a]$ (intuitively, whether they receive a resource depends on who else the policy is comparing them to). Formally, exchanging the order of expectations allows us to write $\text{Eval}^*(\pi)$ as:  
\begin{align*}
    \sum_{i = 1}^N Pr[i \in C_m] \sum_{a \in \mathcal{A}} Pr_{\mathfrak{C}, \pi}[a(i) = a| i \in C_m] \mathbb{E}_s[r(s, a)]
\end{align*}
Since the assignment $\mathfrak{C}$ is random, $ Pr[i \in C_m] = \frac{1}{M}$. Moreoever, in the inner term, conditioning on $i \in C_m$ leaves the other members of $C_m$ distributed uniformly at random. Accordingly, we can estimate $Pr_{\mathfrak{C}, \pi}[a(i) = a| i \in C_m]$ for any policy $\pi$ by drawing repeated samples of the assignment $\mathfrak{C}$ and running $\pi$ to reveal whether $\pi$ would have assigned $a(i) = a$ given the group $C_m$ containing individual $i$. Let $\hat{p}(i, a|\pi)$ denote the fraction of these samples in which $a(i) = a$. A standard IPW estimator for $\textrm{Eval}(\pi)$ is given by
\begin{align}\label{eq:prop-scores}
    \frac{1}{M}\sum_{m=1}^M\sum_{i \in C_m} \frac{\hat{p}(i, a(i)|\pi)}{\hat{p}(i, a(i)|\pi_m)} r(s(i), a(i)).
\end{align}
In a sequential setup ($T>1$), propensity score methods become entirely inapplicable for two reasons. First, standard multi-time step IPW estimators require randomness in the policy, while we assume that policies may be deterministic. We cannot use the alternate approach described above (leveraging randomness in assignments) over multiple time steps, because the marginal probability that individual $i$ receives action $a(i)$ on future steps depends on the state of all other individuals, and we do not have samples of such future states under counterfactual assignments. Second, even if we limited to randomized policies, standard off-policy methods calculate the probability of taking exactly the observed sequence of actions in the observed states. In our case, this requires computing the probability of $\pi$ selecting the \textit{vector} of actions assigned to each individual, i.e, we have a $N$-dimensional action space within each time step. Multi-step IPW estimators are already known to suffer from variance which explodes exponentially in $T$, often rendering them impractical \cite{li2015toward}. In our case, their variance would (in the worst case) scale exponentially in $N$ as well.

\subsection{Main Contribution: Assignment Permutation}

We present a novel approach that counters both challenges to compute a stable, accurate estimator. The key idea behind our estimator is to identify hypothetical trials with counterfactual experimental group assignments, whose reward outcomes can be exactly determined using the given outcomes from the original trial. We leverage the fact that although the state transitions depend only the received allocations, regardless of what policy $\pi$ chooses those allocations.

As a warm-up, consider a single-shot trial $\mathcal{T}$ in which two individuals $i$ and $j$ are assigned to policies $\pi_i$ and $\pi_j$, that make identical resource allocations $a$ to both individuals, yielding outcomes $s_i$ and $s_j$ respectively. Now consider a hypothetical trial $\mathcal{T}^\dag$, run exactly identical to $\mathcal{T}$ except that the assignments of $i$ and $j$ are switched. If in $\mathcal{T}^\dag$, both $i$ and $j$ receive the same allocation $a$ as in $\mathcal{T}$, allocations to other individuals would also remain unaffected, and consequently, all individuals would see identical inputs in both $\mathcal{T}^\dag$ and $\mathcal{T}$. Thus, the actual sample of outcomes $\textbf{s}$ in $\mathcal{T}$ is a sample from the same distribution as that induced by $\mathcal{T}^\dag$.
Generalizing this idea, consider a sequential RCT $\mathcal{T}$, in which a subset of individuals' group assignments are permuted to construct a hypothetical trial $\mathcal{T}^\dag$, which sees new allocations $\textbf{a}^\dag_t$ made at time $t$. 
If an individual experiences sub-trajectories of states $s_{0:t-1}$ and actions $a_{1:t-1}$ till time $t-1$ that are identical in both $\mathcal{T}^\dag$ and $\mathcal{T}$, and if the new allocation $a_t^\dag$ received is also identical to $a_t$, 
then original state sample $s_t$ observed in $\mathcal{T}$, is also a valid sample in $\mathcal{T}^\dag$, drawn from the same distribution, $P^*(\textbf{x}, s_{0:t-1}, a_{1:t}, s_{t})$.
Furthermore, inductively, the entire original state trajectory $s_{0:T}$ of $\mathcal{T}$ can be treated as a valid sample for $\mathcal{T}^\dag$ if $\forall~t~\in [T]$ the input sub-trajectory $s_{0:t-1}$, produces new allocations, $a_t^\dag$ that are identical to $a_t$. 

We exploit this concept to retrospectively check for all such possible reassignments, that would lead to the same sequence of output actions given the same input sub-sequence of the state-action trajectory at all times. The implication is that this allows us to uncover and aggregate outcomes from several such additional `observable counterfactual assignments' (defined below) in estimating the performance of a given test policy. Algorithm~\ref{alg:algorithm-shuffling-general} outlines the idea for a general $M$-arm setting. Later, in Algorithm~\ref{alg:shuffling-index-index} we present an efficient algorithm crafted for handling index-based policies.

\begin{definition}[Observable Counterfactual Assignment]
\label{def:counterfactual-observed}
For an actual assignment $\mathfrak{C}$, we define $\mathfrak{C}^\dag$ to be an observable counterfactual assignment if in a hypothetical trial with assignments $\mathfrak{C}^\dag$, for each $t=1...T$ the actions each policy would assign to each individual are identical to the original actions received, conditioned on the state and action histories ($\textbf{s}_{0:t-1}, \textbf{a}_{1:t-1}$) matching up until time $t-1$.


\end{definition}

\paragraph{Estimation via Assignment Permutation:}

Let $\mathcal{C}^\dag(\mathfrak{C})$ be the set of all  `observable counterfactual assignments' engendered by a single  actual experimental assignment $\mathfrak{C}$. Our proposed estimator averages the outcomes of all such observable counterfactuals:
\begin{equation}
\label{eq:Eval-dag-definition}
    \textrm{Eval}^{\dag}(\pi_m) \coloneqq \frac{\sum_{\mathfrak{C} \in \mathcal{C}^\dag} \textrm{Eval}(\pi_m | \mathfrak{C})}{|\mathcal{C}^\dag|}
\end{equation}

In Theorem~\ref{thm:unbiased}, we show that this is an unbiased estimator for the true expectation $\textrm{Eval}^*$. The main technical step (Lemma \ref{lem:equivalence-relation}) is to show that the that $\mathcal{C}^\dag(\mathfrak{C})$ defines a partition over $\mathcal{C}$ (the set of all possible assignments) where two assignments $\mathfrak{C}_1$, $\mathfrak{C}_2$ lie in the same part if $\mathcal{C}^\dag(\mathfrak{C}_1) = \mathcal{C}^\dag(\mathfrak{C}_2)$. Intuitively, this means that our estimator does not ``overweight" any particular counterfactual assignment; it maintains the equal weight that each has in $\textrm{Eval}^*$.  


\begin{algorithm}
\caption{Estimation through Assignment Permutation}
\label{alg:algorithm-shuffling-general}
\textbf{Input}: States  $\colon \{S_1^{N \times T+1}, \dots, S_M^{N \times T+1}\}$, Actions $\colon \{A_1^{N \times T}, \dots A_M^{N \times T}\}$, Assignment, $\mathfrak{C} \colon \{C_1, \dots, C_M\}$\\
\textbf{Output}: $\textrm{Eval}^\dag$ 
\begin{algorithmic}[1] 
\STATE Compute $\mathcal{C}^\dag$, the set of observable counterfactual assignments of $\mathfrak{C}$.
\STATE Compute $\textrm{Eval}^{\dag}(\pi_m) \coloneqq \frac{\sum_{\mathfrak{C} \in \mathcal{C}^\dag} \textrm{Eval}(\pi_m | \mathfrak{C})}{|\mathcal{C}^\dag|}$
\STATE \textbf{return} $\textrm{Eval}^{\dag}(\pi_m)$
\end{algorithmic}
\end{algorithm}

\subsection{THEORETICAL RESULTS}

In this section, we prove theoretically that $\textrm{Eval}^{\dag}(.)$ 
gives a more accurate estimate because it is unbiased and
simultaneously reduces variance.  
Let $\dag$ be a homogeneous relation on $\mathcal{C}$, defined as: $\dag = \{(\mathfrak{C}_1, \mathfrak{C}_2) \in \mathcal{C} \times \mathcal{C}~\colon~\mathfrak{C}_2 \in \mathcal{C}^\dag(\mathfrak{C}_1)\}$. Intuitively, $\dag$ represents existence of a valid reshuffling to arrive at a counterfactual assignment $\mathfrak{C}_2$ from $\mathfrak{C}_1$.
\begin{restatable}[]{lemma}{equivalenceRelation}\label{lem:equivalence-relation}
The relation $\dag$ is an equivalence relation and the family of sets defined by $\mathcal{C}^\dag(\cdot)$ forms a partition over $\mathcal{C}$.
\end{restatable}
All proofs may be found in the appendix. We leverage this property to prove unbiasedness:

\begin{restatable}[]{theorem}{evalUnbiased}\label{thm:unbiased}
$\textrm{Eval}^\dag(\pi_m)~$ is an unbiased estimate of the expected value of the performance, $\textrm{Eval}^*(\pi_m)~\forall m \in [M]$, defined in equation~\ref{eq:eval-star-definition}. 
i.e. $$\mathbb{E}_{S_m \sim P^*}\mathbb{E}_{\mathfrak{C} \sim \mathcal{C}}[\textrm{Eval}^\dag(\pi_m)]=\textrm{Eval}^*(\pi_m)~\forall m \in [M]$$
\end{restatable}

\begin{restatable}[]{theorem}{varianceReductionEval}\label{thm:variance-reduction-eval}
The sample variance of our estimator, $\textrm{Eval}^\dag(\pi)$ is smaller than the standard estimator, $\textrm{Eval}(\pi)$:
\begin{align*}
&\Var(\textrm{Eval}(\pi)) - \Var(\textrm{Eval}^\dag(\pi)) = \\
  &\frac{1}{|\mathcal{C}|}\cdot \sum_{j \in [\eta]}\Bigg[\sum_{\mathfrak{C} \in \mathcal{P}_j}\textrm{Eval}^2(\pi|\mathfrak{C}) - \frac{\Big(\sum_{\mathfrak{C} \in \mathcal{P}_j}\textrm{Eval}(\pi|\mathfrak{C})\Big)^2}{|\mathcal{P}_j|}\Bigg]
\end{align*}
 $\geq 0$, where $\{\mathcal{P}_1, \dots, \mathcal{P}_\eta\}$ is the partition of $\mathcal{C}$ induced by $\dag$.
\end{restatable}
\textit{Proof Sketch.} We compute the sample variance by first conditioning over the partition $\mathcal{P}_j$ (of the equivalence sets defined by $\dag$) that an instance of an assignment, $\mathfrak{C}$ belongs to and then accounting for the variance stemming from the candidate assignments $\mathfrak{C}$ within the partition.  
Finally, we use the Cauchy-Schwarz inequality to show that the right-hand-side expression in Theorem~\ref{thm:variance-reduction-eval} is non-negative.  $\qedsymbol$

The variance contraction expression of Theorem~\ref{thm:variance-reduction-eval} reduces to zero if and only if $\textrm{Eval}(\pi|\mathfrak{C})$ is identical $\forall \mathfrak{C} \in \mathcal{P}_j, ~\forall~j \in [\eta]$; if different assignments imply different rewards then our estimator exhibits a strict improvement in variance.






\section{Efficient Swapping Algorithm}

Identifying $\mathcal{C}^\dag$ exhaustively in Algorithm~\ref{alg:algorithm-shuffling-general} involves iterating through every possible assignment in $\mathcal{C}$ and running the policy to determine if the assignment belongs to $\mathcal{C}^\dag$. However, the number of possible assignments $|\mathcal{C}|$ grows exponentially with $N$, making full enumeration infeasible. We show how this computational bottleneck can be circumvented for index-based policies with a modified estimator denoted $\textrm{Eval}^\dag_\Upsilon(\cdot)$. This estimator implicitly averages over a subset of the possible permutations in $\mathcal{C}^\dag$, trading off some variance reduction for computational efficiency.




For ease of exposition, here we consider RCTs with two experimental arms ($M=2$) employing allocation policies $\pi_0$ and $\pi_1$ respectively. We aim to estimate $\textrm{Eval}^{*}(\pi_j)_{j=\{0,1\}}$.
Intuitively, instead of working with the space of all possible assignments, we instead consider all individuals participating in the trial and to identify non-overlapping groups of individuals $\{\textbf{G}_k\} \subset (C_0 \cup C_1)$ that satisfy certain desirable properties. Specifically, we intend to find sets of `compatible' individuals such that any subset of individuals within each group $\{\textbf{G}_k\}$ can be mutually swapped to arrive at either an unchanged assignment or a valid observable counterfactual assignment in $\mathcal{C}^\dag(\mathfrak{C})$. Our intention is to compute an estimate by replacing the original reward of every individual $i \in \textbf{G}_k$ by the average of rewards of all individuals in $\textbf{G}_k$, for every such group $\textbf{G}_k$ (justification in Theorem~\ref{thm:efficient-algorithm-works}). 
We identify these groups by checking for two eligibility conditions pertaining to swappability of individuals.

The first eligibility condition for swapping two individuals $i$ and $j$ is that their original allocations $a(t)$ must be identical to each other $~\forall t$, to continue to satisfy the resource constraints in both arms after the swap. To check for this condition, we partition individuals into super-groups $\{\Bar{\textbf{G}}_1,\dots, \Bar{\textbf{G}}_{\kappa}\}$, putting all individuals experiencing the same action vector $a(t) \in \mathcal{A}^T$, in the same super-group, where $\kappa$ denotes the number of such super-groups. All individuals within each $\Bar{\textbf{G}}_k$ satisfy this first eligibility condition for being included in group $\textbf{G}_k$.
For convenience, we let $\phi \colon C_0 \cup C_1 \to [\kappa]$ denote a many-to-one map identifying the super-group $\Bar{\textbf{G}}_{\phi(i)}$ that an individual $i$ belongs to.

The second eligibility condition for swapping an individual is that their new allocation $a^\dag(t)$ under the new policy must be identical to the original $a(t)$, for the same sequence of input states as in the original trial. For each individual $i$, we use a binary-valued variable $\Lambda_i \in \{0,1\}$ to indicate satisfaction of this second condition. 
We introduce and exploit the `index-threshold' property here to verify this condition efficiently. We define an index threshold $\tau_j(t)$ as the smallest value among indices $\Upsilon(t)$ of individuals in $C_j$ at time $t$, that get picked to receive the allocation $a=1$ under policy $\pi_j$. 
To enable efficient computation, we only allow swaps within a group $\textbf{G}_k$ that maintain the index thresholds $\tau_j(t)$ at the same values as the original. 
We implement this constraint by setting $\Lambda_i=0$ for all individuals exactly at the index threshold. Furthermore, for other individuals $i$, $\Lambda_i$ can be cheaply determined by just verifying if the index $\Upsilon^{\pi_j}_i(t)$ lies to the same side of threshold $\tau_j(t)~\forall~t \in [T]$ and for $j \in \{0,1\}$.
To summarize 
\begin{equation}
\label{eq:lambda-cases}
\Lambda_i=
\begin{cases}
  1 &\text{if }\prod_{j=0}^{j=1}(\Upsilon_i^{\pi_j}(t)-\tau_j(t))>0~\forall~t\in[T]\\
  0 &  \text{ otherwise }
\end{cases}
\end{equation}
Intuitively, $\Lambda_i=1$ means that $a^\dag_{1:T}(i) =a_{1:T}(i)$ and indicates that individual $i$ satisfies the second eligibility condition. Taking an intersection of both conditions, we form group $\textbf{G}_k$ by including all individuals $i \in \Bar{\textbf{G}}_k$ that have $\Lambda_i=1$.  
For each group $\textbf{G}_k$, we compute the reward of a representative average individual as: 
\begin{align}
    \label{eq:representative-reward-clique}
    \Tilde{r}_{k} \coloneqq \frac{1}{|\textbf{G}_{k}|} \sum_{ i \in \textbf{G}_{k}}r(S[i], A[i])
\end{align} 

In computation of the final estimate $\textrm{Eval}^{\dag}_\Upsilon(\pi_j)$, we consider all individuals in the arm $C_j$, but replace the reward of swappable individuals among those (i.e. whose $\Lambda_i=1$) by $\Tilde{r}_{\phi(i)}$. We leave the rewards of other individuals unchanged. Finally we compute $\textrm{Eval}^{\dag}_\Upsilon(\pi_j)$ 
by summing up as:   
\begin{align}
    \label{eq:Eval-final-expression}
    \textrm{Eval}^{\dag}_\Upsilon(\pi_j) &= \sum_{i \in C_j}  \Big(\Lambda_i\Tilde{r}_{\phi(i)} + 
    (1-\Lambda_i) r(S[i], A[i])\Big) 
\end{align}

Our theoretical analysis of this estimator establishes that it corresponds to an instance of the general permutation estimator which averages over a subset of the assignments in $\mathcal{C}^\dag$ (instead of the entire set). The main idea is that we can sub-partition $\mathcal{C}^\dag$ into sets with the same value of the index threshold (shown formally in Lemma~\ref{lem:equivalence-relation-efficient-algo}). We denote the part of $\mathcal{C}^\dag$ where the index thresholds are the same as in the actual trial as $\mathcal{C}^\dag_{\Upsilon}$.  Each permutation of individuals within groups $\{\textbf{G}_k\}$ corresponds to an assignment in $\mathcal{C}^\dag_\Upsilon(\mathfrak{C})$, and the final estimator averages over all such assignments:
\begin{restatable}[]{theorem}{efficientAlgorithmWorks}\label{thm:efficient-algorithm-works}
$\textrm{Eval}^\dag_{\Upsilon}(\cdot)$ computed as per Equation~\ref{eq:Eval-final-expression} computes the average of $\textrm{Eval}(\mathfrak{\pi | C})$ over all assignments in $\mathcal{C}^\dag_{\Upsilon}$. i.e. $\textrm{Eval}^\dag_{\Upsilon} (\pi) = \frac{\sum_{\mathfrak{C} \in \mathcal{C}^\dag_{\Upsilon}} \textrm{Eval}(\pi | \mathfrak{C})}{|\mathcal{C}^\dag_{\Upsilon}|}$
\end{restatable}
From this, $\textrm{Eval}^\dag_\Upsilon(\pi_j)$ is easily shown to inherit the desirable properties of the general permutation estimator, e.g., Corollary~\ref{cor:unbiased-eval-dag-upsilon} proves that it is also an unbiased estimator of $\textrm{Eval}^*(\pi_j)$. The tradeoff is a slight sacrifice in variance contraction as it yields smaller partitions $\mathcal{P}_j$ of $\mathcal{C}$ (as defined in Theorem~\ref{thm:variance-reduction-eval}), since we discard assignments with a different threshold. However, working with $\mathcal{C}^\dag_\Upsilon$ enables a computationally efficient algorithm for index policies, avoiding the exponential runtime of the general estimator.

\begin{algorithm}
\caption{Reshuffling between index based policies}
\label{alg:shuffling-index-index}
\textbf{Input}: $\mathfrak{C} \coloneqq \{C_0, C_1\}$, States: $S_0, S_1 \in \mathcal{S}^{N \times T+1}$, Actions: $A_0, A_1 \in \{0,1\}^{N \times T}$, Indexes: $\Upsilon^{\pi_0}, \Upsilon^{\pi_1} \in \mathbb{R}^{2N \times T}$\\
\textbf{Output}: Estimates: $\textrm{Eval}^\dag_\Upsilon$ 
\begin{algorithmic}[1] 
\STATE Group individuals according to action vectors into $\{\Bar{\textbf{G}}_k\}$ and determine $\phi(i)~\forall$ individuals $ i \in C_0 \cup C_1$.
\STATE Determine $\forall t \in [T]$, index thresholds on both arms as: $\Bar{\tau}_j(t) \coloneqq \min \big\{\Upsilon_{\pi_j}(i, t)~| A_j[i,t] =1,~i \in [N] \big\}$
\STATE Determine $\Lambda_i~\forall~$ individuals $i$ according to Equation~\ref{eq:lambda-cases}.
\STATE Compute group $\textbf{G}_k$ as: $\{i \in \Bar{\textbf{G}}_k\ | \Lambda_i =1\}, \forall k \in [\kappa]$
\STATE For each group $\textbf{G}_k$, compute the average reward of a representative individual as per Equation~\ref{eq:representative-reward-clique}. 
\STATE Compute $\textrm{Eval}^\dag_\Upsilon(\pi_j)$ per Equation~\ref{eq:Eval-final-expression} 
\end{algorithmic}
\end{algorithm}

\section{EMPIRICAL EVALUATION}
\label{sec:experiments}
We test our proposed methodology empirically on several datasets: (1) synthetic example (2) semi-synthetic tuberculosis medication adherence monitoring data and (3) real-world field trial data from an intervention for a maternal healthcare. We consider a state space $\mathcal{S} =\{0,1\}$, respectively representing an `undesirable' and a `desirable' state (of health, program engagement, etc.). The action space $\mathcal{A} =\{0,1\}$, denotes `no delivery' or `delivery' of an intervention. We assume a budget constraint, limiting the total number of interventions per time step. The reward function is defined as $r(s_{0:T}, a_{1:T}) = \sum_{t}s_t$, translating to an objective of maximizing the total time spent by individuals in state $s=1$. 


\subsection{Synthetic Dataset}
\label{sec:NR-SH-ID-data}
This setup consists of three types of individuals characterized by their $P-$ matrices. 
$P_1$ and $P_2$ are designed such that it is always optimal to intervene on $P_1$ individuals consistently, whereas intervening on $P_2$ individuals is strictly sub-optimal (details in Appendix~\ref{app:experimental-results-NR-SH-ID-data}). $P_3$ individuals are unaffected by interventions, with transition dynamics independent of the action received. We consider two test policies $\pi_1$ and $\pi_2$, which are designed such that policy $\pi_j$ always chooses individuals of type $P_j$ to intervene on, when available, making $\pi_1$ the optimal policy. We simulate $300~P_1$ individuals, $300~P_2$ individuals and $(300^*\eta~P_3)$ individuals and  set an intervention budget of $300$ per timestep for $T=20$ timesteps. We measure the performance lift of $\pi_1$ against $\pi_2$ as: $\Delta \coloneqq \Delta(\pi_1, \pi_2)=\textrm{Eval}(\pi_1)-\textrm{Eval}(\pi_2)$.

\begin{figure}
     \centering
     \begin{subfigure}[b]{0.43\linewidth}
         \centering
         \includegraphics[width=\textwidth]{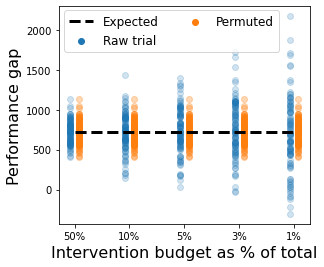}
         \caption{Total engagements}
         \label{fig:toy-scatter}
     \end{subfigure}
     \begin{subfigure}[b]{0.43\linewidth}
         \centering
         \includegraphics[width=\textwidth]{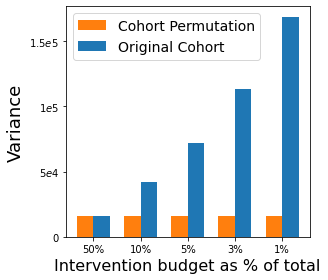}
         \caption{Sample Variance}
         \label{fig:toy-variance}
     \end{subfigure}
        \caption{Estimates and sample variance in synthetic domain.}
        \label{fig:performance-gap-comparison-total}
\end{figure}




In Figure~\ref{fig:toy-scatter}, we vary the budget on the x-axis. 
Each blue dot in the scatter plot shows one independent RCT instance and measures the raw difference in rewards $\Delta$ on the y-axis. Applying assignment permutation maps each blue dot to an orange dot. The black dashed line marks the expected value of the performance lift. Visually, both colors are centered on the black line, as both estimators are unbiased. The assignment-permuted estimates lie closer to the expected value than the raw estimates, indicating a smaller sample variance. Quantitatively, we measure the sample variance of $\Delta$ on the y-axis in Figure~\ref{fig:toy-variance}. Variance reduces sharply upon applying assignment permutation -- for instance, at a budget level of $3\%$, our approach cuts the variance by $7 \times$, from $11.3\times10^4$ to $1.6\times10^4$. The intuition is that both $\pi_1$ and $\pi_2$ overlap in their decision to not intervene upon $P_3$ individuals. However, their final rewards are based partly on which $P_3$ individual gets (randomly) assigned to which group, independent of the underlying policies. Assignment permutation counters this randomness by  averaging over alternate assignments of the $P_3$ individuals.

\subsection{Semi-synthetic evaluation with tuberculosis dataset}

We use real tuberculosis medication adherence monitoring data, consisting of daily records of patients in Mumbai, India, obtained from \cite{killian2019learning} and simulate patient behavior by estimating the $P$ matrix. More details can be found in the appendix.

\begin{figure}
     \centering
     \begin{subfigure}[b]{0.55\linewidth}
         \centering
         \includegraphics[width=\textwidth]{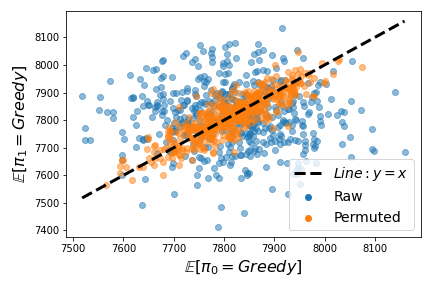}
         \caption{}
         \label{fig:multi-step-greedy-greedy-tb-scatter}
     \end{subfigure}
     \begin{subfigure}[b]{0.32\linewidth}
         \centering
         \includegraphics[width=\textwidth]{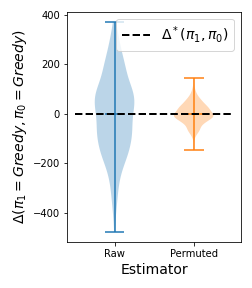}
         \caption{}
         \label{fig:multi-step-greedy-greedy-tb-violin}
     \end{subfigure}
        \caption{Multi-step setting. (a) The permuted estimates (orange) are closer to the true expectation (black line) than the raw estimates (blue) (b) Assignment permutation reduces variance.}
        \label{fig:multi-step-greedy-greedy-tb}
\end{figure}

\begin{figure*}[t]
\label{fig:single-shot}
    \centering
    \begin{minipage}{.44\textwidth}  
     \centering
     \begin{subfigure}[b]{0.57\linewidth}
         \centering
         \includegraphics[width=\textwidth]{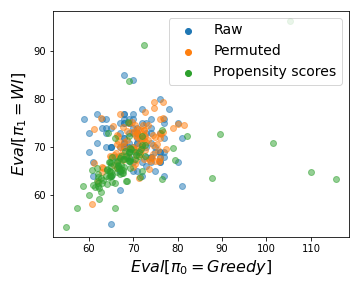}
         \caption{}
         \label{fig:single-shot-scatter}
     \end{subfigure}
     \begin{subfigure}[b]{0.39\linewidth}
         \centering
         \includegraphics[width=\textwidth]{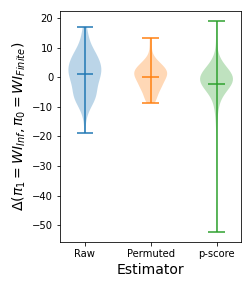}
         \caption{}
         \label{fig:single-shot-iolin}
     \end{subfigure}
     
     \caption{Illustration of results for single-step setting}
    \end{minipage}%
    \hfill
    \begin{minipage}{0.55\textwidth}
    \textit{Table 1:} Sample variance in Measured Performance Lift
    \vspace{2mm}
\label{variance-table}
\begin{center}
\begin{small}
\begin{sc}

        \begin{tabular}{cccccccc}
\toprule
T & $B$ & $\pi_1~\textrm{v}~\pi_0$                        & raw & permuted & ipw & $n$-val \\
\midrule
$1$&    3\%     &$\pi_{\textsc{wi}}$ v $\pi_{\textsc{gr}}$     & 49.09 & 4.94 &\textbf{ 0.48}   & 9 \\
$1$&    10\%    &$\pi_{\textsc{wi}}$ v $\pi_{\textsc{gr}}$    & 49.86 & 15.11 & \textbf{6.66}   & 3 \\
$1$&    25\%    &$\pi_{\textsc{wi}}$ v $\pi_{\textsc{gr}}$    & 49.45 & \textbf{19.94 }& 78.12 & 2 \\
$10$&   3\%     &$\pi_{\textsc{wi}}$ v $\pi_{\textsc{wi}}$     & 2381 & \textbf{916} & NA      & 3 \\
$10$&   3\%     &$\pi_{\textsc{wi}}$ v $\pi_{\textsc{gr}}$     & 2348 & \textbf{728} & NA      & 4 \\
$10$&   3\%     &$\pi_{\textsc{gr}}$ v $\pi_{\textsc{gr}}$     & 26356 & \textbf{1860} & NA    & 13 \\
$10$&   10\%     &$\pi_{\textsc{gr}}$ v $\pi_{\textsc{gr}}$     & 25983 & \textbf{3808} & NA    & 7 \\
$10$&   25\%     &$\pi_{\textsc{gr}}$ v $\pi_{\textsc{gr}}$     & 23619 & \textbf{5477} & NA    & 5 \\
\bottomrule
\end{tabular}
\end{sc}
\end{small}
\medskip

\end{center}
\vskip -0.1in
    \end{minipage}%
\end{figure*}



We consider two policies: a ``Whittle index" policy \cite{whittle1988restless} that attempts to maximize long-run reward, and a greedy policy which optimizes an estimate of next-step reward. We simulate $N = 1000$ patients in each arm and vary the budget constraint. We consider both the multi-step setting ($T = 10$) and single step ($T = 1$).  We compare three estimation methods: ``Raw" is the naive average of outcomes in each arm, ``Permuted" our proposed estimator, and ``IPW" the inverse propensity estimator from Section 4.1 (available only for $T=1$). In practice, we find it necessary to trim propensity scores for IPW \cite{zhou2020propensity} to the range [0.01, 0.99], since extreme values lead to very large variance. This introduces a slight bias, visible in Figure \ref{fig:single-shot-iolin}.

Table~\ref{variance-table} shows the sample variance of estimates returned by each method. For unbiased methods (Raw, Permuted) the variance is also their mean squared error; this holds approximately for IPW due to trimming. Table~\ref{variance-table} includes an additional column labeled `$n-$value'. To benchmark the improvement produced by our method, this gives the minimum number of \textit{independent} RCTs that would need to be run (and averaged over) to match the sample variance achieved by assignment permutation (computed by simulation). 

For all comparisons and parameter settings, we find that our assignment permutation estimator produces a substantial improvement in variance. Indeed, achieving a comparably precise estimate using the naive raw estimator would require running anywhere from 2 to \textit{13} independent RCTs. This underscores the importance of variance reduction -- running RCTs is hugely costly and assigns many individuals to suboptimal policies; improved analysis allows us to draw comparably precise conclusions at dramatically lower cost.

Figure~\ref{fig:multi-step-greedy-greedy-tb} illustrates this improvement in a single example where both trial arms are the Greedy policy and so the expected difference in rewards is exactly zero (with $B = 3\%$ and $T = 10$). Each dot in Figure~\ref{fig:multi-step-greedy-greedy-tb-scatter} corresponds to a single instance of a trial, with the x- and y-axes giving total engagements in the two arms. The black dashed line ($y=x$) denotes the expectation, $\Delta =0$. The blue dots, representing raw measurements, have a wider spread around the black line than the orange dots obtained via assignment permutation. Figure~\ref{fig:multi-step-greedy-greedy-tb-violin}, shows a violin plot of the sample difference in rewards between the two arms.  
Both violins are centered on the zero line, reflecting that the estimators are unbiased. The violin corresponding to assignment permutation is more compact, indicating lower sample variance. 

In the single-step setting, the IPW estimator returns mixed results: it has the best variance for small values of the budget, but actually performs \textit{worse} than the naive raw estimator for larger budgets. Essentially, the overlap between two policies becomes smaller as the budget increases because they agree only on the few highest-priority individuals. Low overlap translates into extreme propensity scores, inflating variance. Figures \ref{fig:single-shot-iolin} and \ref{fig:single-shot-scatter} show an illustration, where IPW often produces an improvement but is susceptible to large outliers. However, when overlap is high and we operate only in the single-step setting, IPW can be a valuable option.

\subsection{Case study: Real-world Trial}

Our method is directly applicable to real-world settings; we show this by considering an actual large-scale RCT reported in \citep{mate2022field} evaluating a Restless Multi-Armed Bandit-based algorithm for  resource allocation in a maternal and child healthcare. The data consists of 23,000 real-world beneficiaries, randomly split between three groups for the trial: RMAB algorithm, baseline algorithm and a control group, which sees no interventions. Real-world health workers delivered interventions recommended by the algorithms. We consider the performance lift of the RMAB algorithm in improving engagement with the program in comparison to the control group and apply our proposed permutation algorithm to the originally reported raw results. Figure~\ref{fig:armman-april21-reshuffled-s-agnostic} (left) plots the total engagement numbers on the y-axis as a function of time (in weeks) on the x-axis. Figure~\ref{fig:armman-april21-reshuffled-s-agnostic} (right) computes the lift provided by the RMAB algorithm, as defined in \cite{mate2022field} on the y-axis. Our findings suggest that the performance lift of RMAB algorithm is larger than originally reported and by week 7, RMAB is estimated to prevent $815$ engagement drops, vs the originally reported $622$. 
Since this is real data, the true values are unknown. However, this case study provides evidence that the variance reduction provided by our estimator can be significant in practice.
\begin{figure}
     \centering
         \includegraphics[width=1.6in]{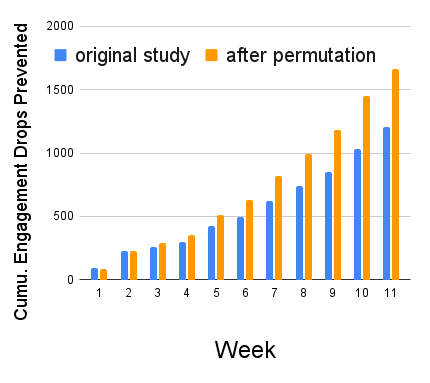}
         \includegraphics[width=1.6in]{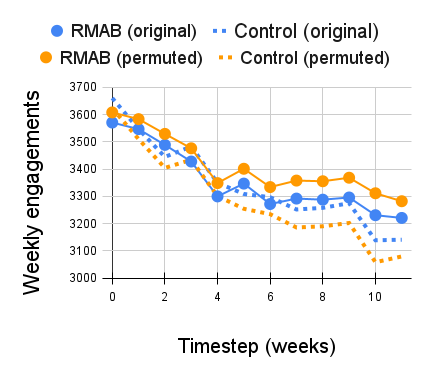}
    \caption{Impact of permutation estimator on real-world data.}
        \label{fig:armman-april21-reshuffled-s-agnostic}
\end{figure}
\section{CONCLUSION}
 We address the critical gap of mitigating the error in evaluation of resource allocation policies through RCTs. We propose a new estimator using a novel concept based on the idea of retrospective reassignment of participants to experimental arms. We prove that our estimator is unbiased while simultaneously reducing sample variance, and hence reduces error. Through empirical tests on multiple data sets -- including a real-world dataset in a socially critical domain -- we show that our approach cuts error by as much as $70\%$ and from a single given RCT, can achieve benefits equivalent of running upto $13$ \textit{independent} RCTs in parallel.
\bibliography{references}
\bibliographystyle{icml2022}

\newpage
\appendix
\onecolumn
\section{Complete Proofs to Theoretical Results}
\label{app:theoretical-results}

\subsection{Proof of Lemma~\ref{lem:equivalence-relation}}
\label{app:theoretical-results-equivalence-relation}
\equivalenceRelation
\begin{proof}
To prove that $\dag$ is an equivalence relation, we show that it is reflexive, symmetric and transitive. $\dag$ is reflexive because $\forall \mathfrak{C} \in \mathcal{C}, \mathfrak{C} \in \mathcal{C}^\dag(\mathfrak{C})$ by definition. Furthermore, $\dag$ is also trivially symmetric because if $\mathfrak{C}_2 \in \mathcal{C}^\dag(\mathfrak{C}_1)$, then by definition, the allocations received by all individuals at all times are identical under both $\mathfrak{C}_1$ and $\mathfrak{C}_2$. Hence $\mathfrak{C}_1 \in \mathcal{C}^\dag(\mathfrak{C}_2)$. Finally, $\dag$ is also transitive because if all allocations received by all individuals at all times are identical in $\mathfrak{C}_1$ and $\mathfrak{C}_2$ as well as in $\mathfrak{C}_2$ and $\mathfrak{C}_3$, that means the allocations are also identical in $\mathfrak{C}_1$ and $\mathfrak{C}_3$. Thus formally, if $\mathfrak{C}_2 \in \mathcal{C}^\dag(\mathfrak{C}_1)$ and $\mathfrak{C}_3 \in \mathcal{C}^\dag(\mathfrak{C}_2)$, then $\mathfrak{C}_3 \in \mathcal{C}^\dag(\mathfrak{C}_1)$. Thus $\dag$ is an equivalence relation over $\mathcal{C}$ and consequently, partitions $\mathcal{C}$ into a family of equivalence classes $\mathcal{C}^\dag(\cdot)$ such that every element $\mathfrak{C} \in \mathcal{C}$ lies in exactly one partition \cite{enderton1977elements}. 
\end{proof}

\subsection{Proof of Theorem~\ref{thm:unbiased}}
\label{app:theoretical-results-eval-unbiased}
\evalUnbiased*
\begin{proof} \begin{align*}
     \textrm{Eval}^*(\pi_m)= & \mathbb{E}_{S_m \sim P^*} \Big[\mathbb{E}_{\mathfrak{C} \sim \mathcal{C}}\big[\textrm{Eval}(\pi_m)\big]\Big] \\
    =  \mathbb{E}_{S_m \sim P^*}& \Big[ \sum_{\mathfrak{C}~\in~ \mathcal{C}}\mathrm{Prob}[\mathfrak{C}]\cdot\textrm{Eval}(\pi_m|\mathfrak{C}) \Big] \\
    =  \mathbb{E}_{S_m \sim P^*}\Big[&\frac{1}{|\mathcal{C}|}\sum_{\mathfrak{C}~\in~ \mathcal{C}}\textrm{Eval}(\pi_m|\mathfrak{C}) \Big] &(~\textrm{$\because$ all $\mathfrak{C}$ equally likely})\\
    =  \mathbb{E}_{S_m \sim P^*}&\bigg[\frac{1}{|\mathcal{C}|}\Big[\sum_{\mathfrak{C}~\in~ \mathcal{P}_1}\textrm{Eval}(\pi_m|\mathfrak{C}) + \dots + \sum_{\mathfrak{C}~\in~ \mathcal{P}_{\eta}}\textrm{Eval}(\pi_m|\mathfrak{C}) \Big] \bigg]\\
    \textrm{where }  \{\mathcal{P}_1, & \dots, \mathcal{P}_{\eta}\} \textrm{ defines partition of $\mathcal{C}$ induced by $\dag$.} \\ 
    =  \mathbb{E}_{S_m \sim P^*}&\bigg[\sum_{j \in [\eta]}\frac{|\mathcal{P}_j|}{|\mathcal{C}|}\cdot\frac{1}{|\mathcal{P}_j|}.\Big[\sum_{\mathfrak{C}~\in~ \mathcal{P}_j}\textrm{Eval}(\pi_m|\mathfrak{C}) \Big] \bigg]\\
    =  \mathbb{E}_{S_m \sim P^*}&\bigg[\sum_{j \in [\eta]}\frac{|\mathcal{P}_j|}{|\mathcal{C}|}\cdot\Big[\textrm{Eval}^\dag(\pi_m|\mathfrak{C}) \Big]\bigg]  &(\forall \mathfrak{C} \in \mathcal{P}_j)\\
    =  \mathbb{E}_{S_m \sim P^*}&\bigg[\sum_{j \in [\eta]} \textrm{Prob}[\mathcal{P}_j]\cdot\Big[\textrm{Eval}^\dag(\pi_m|\mathfrak{C}) \Big]\bigg]&(\forall \mathfrak{C} \in \mathcal{P}_j) \\
    =  \mathbb{E}_{S_m \sim P^*}&\bigg[\sum_{j \in [\eta]} \sum_{\mathfrak{C} \in \mathcal{P}_j}\textrm{Prob}[\mathfrak{C}]\cdot\Big[\textrm{Eval}^\dag(\pi_m|\mathfrak{C}) \Big]\bigg] &\big(\because \textrm{Prob}[\mathcal{P}_j] = \sum_{\mathfrak{C} \in \mathcal{P}_j} \textrm{Prob}[\mathfrak{C}] \big)\\
    =  \mathbb{E}_{S_m \sim P^*}&\mathbb{E}_{\mathfrak{C} \sim \mathcal{C}}[\textrm{Eval}^\dag(\pi_m)]
\end{align*} 
\end{proof}


\subsection{Proof of Theorem~\ref{thm:variance-reduction-eval}}
\label{app:theoretical-results-variance-thm}
\varianceReductionEval*
\begin{proof}
    We compute the sample variance by first conditioning over the partition $\mathcal{P}_j$ (of the equivalence sets defined by $\dag$) that an instance of an assignment, $\mathfrak{C}$ belongs to and then accounting for the variance stemming from the candidate assignments $\mathfrak{C}$ within the partition. Thus we get: 
    \begin{align}
        \Var(\textrm{Eval}(\pi))  & = \frac{1}{|\mathcal{C}|}\sum_{\mathfrak{C} \in \mathcal{C}} \big(\textrm{Eval}(\pi|\mathfrak{C})-\textrm{Eval}^*(\pi)\big)^2 \nonumber \\ 
        & = \frac{1}{|\mathcal{C}|}\sum_{\mathfrak{j} \in [\eta]} \sum_{\mathfrak{C} \in \mathcal{P}_j} \big(\textrm{Eval}(\pi|\mathfrak{C})-\textrm{Eval}^*(\pi)\big)^2 \nonumber \\ 
        \label{eq:eval-final-expansion}
        & = \frac{1}{|\mathcal{C}|}\sum_{\mathfrak{j} \in [\eta]} \bigg( \sum_{\mathfrak{C} \in \mathcal{P}_j} 
        \big(\textrm{Eval}(\pi|\mathfrak{C})\big)^2 -2~\textrm{Eval}^*(\pi)\sum_{\mathfrak{C} \in \mathcal{P}_j}\textrm{Eval}(\pi|\mathfrak{C}) + |\mathcal{P}_j|\big(\textrm{Eval}^*(\pi)\big)^2 \bigg)
    \end{align}
    Similarly, we compute the variance of our estimator $\textrm{Eval}^\dag$ as: 
    \begin{align}
        \Var(\textrm{Eval}^\dag(\pi))  & = \frac{1}{|\mathcal{C}|}\sum_{\mathfrak{j} \in [\eta]} \sum_{\mathfrak{C} \in \mathcal{P}_j} \big(\textrm{Eval}^\dag(\pi|\mathfrak{C})-\textrm{Eval}^*(\pi)\big)^2 \nonumber \\
        & = \frac{1}{|\mathcal{C}|}\sum_{\mathfrak{j} \in [\eta]} \Bigg(|\mathcal{P}_j|\cdot\bigg\{\frac{\sum_{\mathfrak{C} \in \mathcal{P}_j}\textrm{Eval}(\pi|\mathfrak{C})}{|\mathcal{P}_j|}- \textrm{Eval}^*(\pi)\bigg\}^2\Bigg) \nonumber \\
        & = \frac{1}{|\mathcal{C}|}\sum_{\mathfrak{j} \in [\eta]} \Bigg(|\mathcal{P}_j|\cdot\bigg\{\bigg(\frac{\sum_{\mathfrak{C} \in \mathcal{P}_j}\textrm{Eval}(\pi|\mathfrak{C})}{|\mathcal{P}_j|}\bigg)^2 - 2~\textrm{Eval}^*(\pi)~\Big(\frac{\sum_{\mathfrak{C} \in \mathcal{P}_j}\textrm{Eval}(\pi|\mathfrak{C})}{|\mathcal{P}_j|}\Big) + \Big(\textrm{Eval}^*(\pi)\Big)^2\bigg\}\Bigg) \nonumber \\
        \label{eq:eval-dag-final-expansion}
        & = \frac{1}{|\mathcal{C}|}\sum_{\mathfrak{j} \in [\eta]} \Bigg\{\frac{\Big(\sum_{\mathfrak{C} \in \mathcal{P}_j}\textrm{Eval}(\pi|\mathfrak{C})\Big)^2}{|\mathcal{P}_j|} -2~\textrm{Eval}^*(\pi)\sum_{\mathfrak{C} \in \mathcal{P}_j}\textrm{Eval}(\pi|\mathfrak{C}) + |\mathcal{P}_j|\big(\textrm{Eval}^*(\pi)\big)^2\Bigg\} 
    \end{align}
   Subtracting expression in Equation~\ref{eq:eval-dag-final-expansion} from the expression in Equation~\ref{eq:eval-final-expansion} gives: 
   \begin{align}
        \label{eq:variance-contraction-final-expression-appendix}
       \Var(\textrm{Eval}(\pi))- \Var(\textrm{Eval}^\dag(\pi)) = \frac{1}{|\mathcal{C}|}\sum_{\mathfrak{j} \in [\eta]} \Bigg( \sum_{\mathfrak{C} \in \mathcal{P}_j} 
        \big(\textrm{Eval}(\pi|\mathfrak{C})\big)^2 - \frac{\Big(\sum_{\mathfrak{C} \in \mathcal{P}_j}\textrm{Eval}(\pi|\mathfrak{C})\Big)^2}{|\mathcal{P}_j|} \Bigg)
   \end{align}

We can show that the expression for variance contraction derived in Equation~\ref{eq:variance-contraction-final-expression-appendix} is non-negative as a direct consequence of the Cauchy-Schwarz inequality. The Cauchy-Schwarz inequality states that for two vectors $\textbf{u}$ and $\textbf{v}$, $|\langle\textbf{u},\textbf{v}\rangle|^2 \le \langle \textbf{u},\textbf{u}\rangle \cdot \langle \textbf{v},\textbf{v}\rangle$, where $\langle \cdot , \cdot \rangle$ denotes the inner product. Setting $\textbf{u} = \underbrace{\Big[\textrm{Eval}(\pi |\mathfrak{C}_1), \dots, \textrm{Eval}(\pi |\mathfrak{C}_{|\mathcal{P}_j|})\Big]}_{|\mathcal{P}_j| \textrm{entries}})$ and $\textbf{v} = \underbrace{\Big[1, \dots, 1\Big]}_{|\mathcal{P}_j|~1's}$ yields the desired result: $\Var(\textrm{Eval}(\pi))- \Var(\textrm{Eval}^\dag(\pi)) \ge 0$.
\end{proof}

\subsection{Efficient Algorithm}

\begin{lemma}
\label{lem:equivalence-relation-efficient-algo}
 The relation $\dag_\Upsilon$ is an equivalence relation over both $\mathcal{C}$ as well as each set in the family $\mathcal{C}^\dag(\cdot)$ and the family of sets defined by $\mathcal{C}^\dag_\Upsilon(\cdot)$ forms a partition over $\mathcal{C}$.
\end{lemma}
\begin{proof}
Similar to Lemma~\ref{lem:equivalence-relation}, we prove that $\dag_\Upsilon$ is an equivalence relation by showing that it is reflexive, symmetric and transitive. $\dag_\Upsilon$ is reflexive because $\forall \mathfrak{C} \in \mathcal{C}, \mathfrak{C} \in \mathcal{C}^\dag_\Upsilon(\mathfrak{C})$ by definition. Furthermore, $\dag$ is also trivially symmetric because if $\mathfrak{C}_2 \in \mathcal{C}^\dag(\mathfrak{C}_1)$, then by definition, the index thresholds, as well as the allocations received by all individuals at all times, are identical under both $\mathfrak{C}_1$ and $\mathfrak{C}_2$. Hence $\mathfrak{C}_1 \in \mathcal{C}^\dag(\mathfrak{C}_2)$. Finally, $\dag$ is also transitive because if all index threshold and allocations received by all individuals at all times are identical in $\mathfrak{C}_1$ and $\mathfrak{C}_2$ as well as in $\mathfrak{C}_2$ and $\mathfrak{C}_3$, that means the same are also identical in $\mathfrak{C}_1$ and $\mathfrak{C}_3$. Thus formally, if $\mathfrak{C}_2 \in \mathcal{C}^\dag(\mathfrak{C}_1)$ and $\mathfrak{C}_3 \in \mathcal{C}^\dag(\mathfrak{C}_2)$, then $\mathfrak{C}_3 \in \mathcal{C}^\dag(\mathfrak{C}_1)$. Thus $\dag_\Upsilon$ is an equivalence relation over $\mathcal{C}$ and consequently, partitions $\mathcal{C}$ into a family of equivalence classes $\mathcal{C}^\dag(\cdot)$ such that every element $\mathfrak{C} \in \mathcal{C}$ lies in exactly one partition \cite{enderton1977elements}. Similar reasoning also shows that $\dag_\Upsilon$ is an equivalence relation over each set in the family $\mathcal{C}^\dag(\cdot)$.
\end{proof}

\begin{corollary}
\label{cor:unbiased-eval-dag-upsilon}
    $\textrm{Eval}^\dag_{\Upsilon}(\pi_m)~$ is an unbiased estimate of the expected value of the performance, $\textrm{Eval}^*(\pi)$, defined in equation~\ref{eq:eval-star-definition}. 
    i.e. $\mathbb{E}_{S_m \sim P^*}\mathbb{E}_{\mathfrak{C} \sim \mathcal{C}}[\textrm{Eval}^\dag(\pi_m)]=\textrm{Eval}^*(\pi_m)~\forall m \in [M]$
\end{corollary}
\begin{proof}
    Using Lemma~\ref{lem:equivalence-relation-efficient-algo}, we apply similar arguments as Theorem~\ref{thm:unbiased} on the parition defined by $\dag_{\Upsilon}$ to show that $\textrm{Eval}^\dag_{\Upsilon}$ yields an unbiased estimate.
\end{proof}

\efficientAlgorithmWorks*
\begin{proof}
    The key to showing that the two are equivalent is in interpreting the summation of (modified) rewards over individuals in the view of average group rewards over assignments. Mathematically, starting from the definition of $\textrm{Eval}^\dag_{\Upsilon}(\pi_m)~$, the key lies in moving the summation operation over assignments in $\mathcal{C}^\dag_\Upsilon$ from outside the $\textrm{Eval}^\dag_\Upsilon()$ term to inside, applying it individually on each contributing participant. Formally, we can rewrite the expression of $\textrm{Eval}^\dag_{\Upsilon}(\pi_m)~$ as:
\begin{align}
  \textrm{Eval}^{\dag}_{\Upsilon}(\pi_j) &= \frac{\sum_{\mathfrak{C} \in \mathcal{C}^\dag_{\Upsilon}} \textrm{Eval}(\pi_j | \mathfrak{C})}{|\mathcal{C}^\dag_{\Upsilon}|} \\
  &=  \frac{\sum_{\mathfrak{C} \in \mathcal{C}^\dag_{\Upsilon}} \sum_{i \in C_j}r(S[i], A[i])}{|\mathcal{C}^\dag_{\Upsilon}|} \\
  &=  \sum_{\mathfrak{C} \in \mathcal{C}^\dag_{\Upsilon}} \frac{ \sum_{k \in \kappa}\sum_{i \in \textbf{G}_k} \mathbbm{1}_{\{i \in C_j\}} \cdot r(S[i], A[i]) + \sum_{i \in C_j} (1-\Lambda_i)\cdot r(S[i], A[i])}{|\mathcal{C}^\dag_{\Upsilon}|} \\
  & \textrm{(splitting the summation over groups $\textbf{G}_k$ and other individuals that can't be swapped)}\\
  &=  \sum_{k \in \kappa} \sum_{i \in \textbf{G}_k} \Bigg[\sum_{\mathfrak{C} \in \mathcal{C}^\dag_{\Upsilon}} \frac{ \mathbbm{1}_{\{i \in C_j\}}\cdot r(S[i], A[i])}{|\mathcal{C}^\dag_{\Upsilon}|} \Bigg] + \sum_{i \in C_j} (1-\Lambda_i)\cdot r(S[i], A[i])\\ 
  &=  \sum_{k \in \kappa} \sum_{i \in \textbf{G}_k} \Bigg[\sum_{\mathfrak{C} \in \mathcal{C}^\dag_{\Upsilon}} \frac{ \mathbbm{1}_{\{i \in C_j\}}}{|\mathcal{C}^\dag_{\Upsilon}|} \cdot  r(S[i], A[i]) \Bigg] + \sum_{i \in C_j} (1-\Lambda_i)\cdot r(S[i], A[i])\\
  &=  \sum_{k \in \kappa} \sum_{i \in \textbf{G}_k} \Bigg[\textrm{Pr}(i \in C_j | \mathcal{C}^\dag_{\Upsilon} )\cdot  r(S[i], A[i]) \Bigg] + \sum_{i \in C_j} (1-\Lambda_i)\cdot r(S[i], A[i]) \\
    &=  \sum_{k \in \kappa} \sum_{i \in \textbf{G}_k} \Bigg[\frac{ |\{\iota : \iota \in (\textbf{G}_{\phi(i)} \cap C_j)\}|}{|\textbf{G}_{\phi(i)}|}\cdot  r(S[i, 0:T], A[i,1:T]) \Bigg] + \sum_{i \in C_j} (1-\Lambda_i)\cdot r(S[i], A[i])\\
    &=  \sum_{k \in \kappa} |\{\iota : \iota \in (\textbf{G}_{\phi(i)} \cap C_j)\}| \cdot \sum_{i \in \textbf{G}_k} \Bigg[\frac{1}{|\textbf{G}_{\phi(i)}|}\cdot  r(S[i, 0:T], A[i,1:T]) \Bigg] + \sum_{i \in C_j} (1-\Lambda_i)\cdot r(S[i], A[i])\\
    &=  \sum_{k \in \kappa} |\{\iota : \iota \in (\textbf{G}_{\phi(i)} \cap C_j)\}| \cdot  \Tilde{r}_{k} + \sum_{i \in C_j} (1-\Lambda_i)\cdot r(S[i], A[i])\\
    &=  \sum_{i \in C_j} \Tilde{r}_{\phi(i)}\cdot \Lambda_i + \sum_{i \in C_j} (1-\Lambda_i)\cdot r(S[i], A[i], 1:T)
  \label{eq:summation-over-individuals}
\end{align}


\end{proof}

\section{Casting Resource Allocation Policies as Index Policies}
\label{app:index-based-policies}

\paragraph{Control:} A control group that sees no interventions can be handled by using any randomly generated index matrix with finite entries. Setting the index threshold $\Upsilon_i = \infty~\forall~i$ ensures that no individual assigned the control policy gets picked for intervention. 

\paragraph{Round Robin:} Common policies such as `round robin', that operate by selecting individuals cyclically for intervention in a set order can also be represented as index policies. The index for each individual at each time step, can be determined in two stages. First, we consider the feature used for ranking the $N$ individuals and we start by setting $\Upsilon_i (t) \coloneqq r$, for $r \in \{1, \dots N\}$ where $r$ denotes the priority rank of the individual (highest rank picked first). Next, each time an individual receives an action $a=1$, we want to push them at the bottom of the queue, so we subtract $N$ from their index for all future timesteps, repeating this process for each instance of $a=1$. 

\section{Additional Experimental Results}
\label{app:experimental-results}

\begin{table}
\caption{Sample variance in Measured Performance Lift}
\label{variance-table}
\vskip 0.15in
\begin{center}
\begin{small}
\begin{sc}
\begin{tabular}{cccccccc}
\toprule
T & $B$ & $\pi_1~\textrm{v}~\pi_0$                        & raw & permuted & ipw & $n$-val \\
\midrule
$1$&    3\%     &$\pi_{\textsc{wi}}$ v $\pi_{\textsc{gr}}$     & 49.09 & 4.94 &\textbf{ 0.48}   & 9 \\
$1$&    10\%    &$\pi_{\textsc{wi}}$ v $\pi_{\textsc{gr}}$    & 49.86 & 15.11 & \textbf{6.66}   & 3 \\
$1$&    25\%    &$\pi_{\textsc{wi}}$ v $\pi_{\textsc{gr}}$    & 49.45 & \textbf{19.94 }& 78.12 & 2 \\
$10$&   3\%     &$\pi_{\textsc{wi}}$ v $\pi_{\textsc{wi}}$     & 2381 & \textbf{916} & NA      & 3 \\
$10$&   3\%     &$\pi_{\textsc{wi}}$ v $\pi_{\textsc{gr}}$     & 2348 & \textbf{728} & NA      & 4 \\
$10$&   3\%     &$\pi_{\textsc{gr}}$ v $\pi_{\textsc{gr}}$     & 26356 & \textbf{1860} & NA    & 13 \\
$10$&   10\%     &$\pi_{\textsc{gr}}$ v $\pi_{\textsc{gr}}$     & 25983 & \textbf{3808} & NA    & 7 \\
$10$&   25\%     &$\pi_{\textsc{gr}}$ v $\pi_{\textsc{gr}}$     & 23619 & \textbf{5477} & NA    & 5 \\
\bottomrule
\end{tabular}
\end{sc}
\end{small}
\end{center}
\vskip -0.1in
\end{table}

\subsection{Synthetic Data Generation in Section~\ref{sec:NR-SH-ID-data} Explained}
\label{app:experimental-results-NR-SH-ID-data}
\begin{figure}[H]
    \centering
    \includegraphics[width=0.3\linewidth]{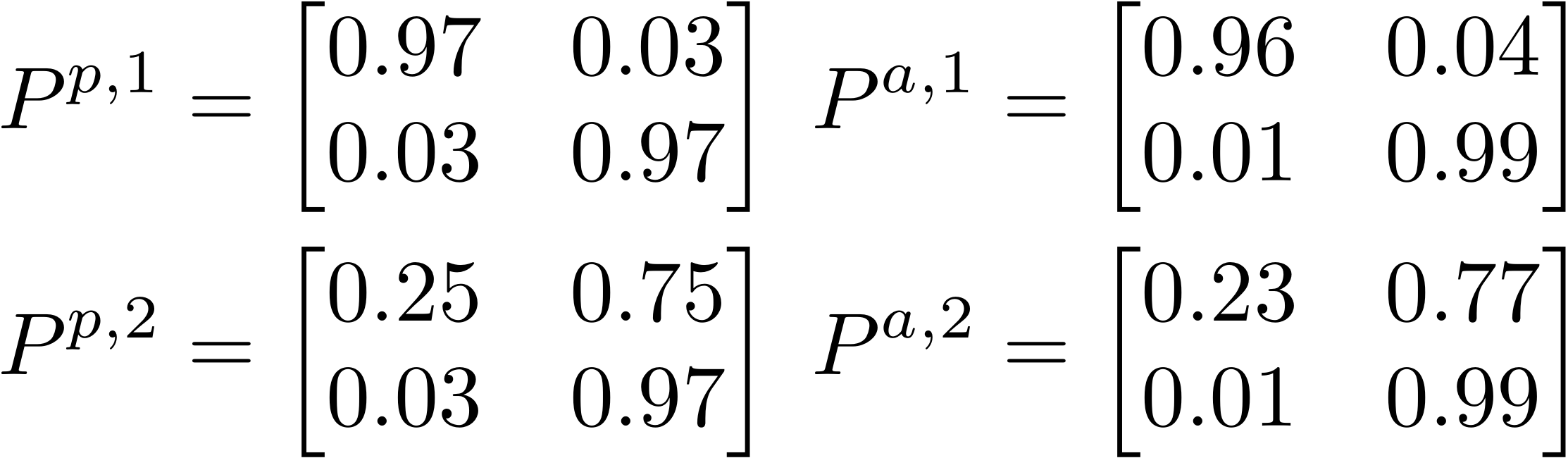}
    \caption{Probability values forming the matrix $P_1$ and $P_2$}
    \label{fig:Synthetic-NR-SH-probs}
\end{figure}
 We reproduce the transition probabilities $P_1$ and $P_2$ used in our simulation, adopted from \cite{mate2020collapsing} in Figure~\ref{fig:Synthetic-NR-SH-probs}. Each $P$ comprises of a set of probabilities under each of the two actions ($a=0$, denoted as `p', for passive and $a=1$ denoted as `a', for active). 

Intuition is that $P_1$ has a very small $P^a_{0,1}$ and $P^p_{0,1}$ and is thus difficult to revive once it enters state $s=0$, even with an intervention ($a=1$), making it important to keep intervening to stop the individual from ever entering $s=0$.
On the other hand, $P_2$ has a large $P^p_{0,1}$, making it self-correcting, meaning the individual is likely to return to $s=1$ quickly even without intervention.
\subsection{Single-shot RCTs}
\label{app:experimental-results-single-shot}

\begin{figure}[H]
    \centering
    \includegraphics[width=0.35\textwidth]{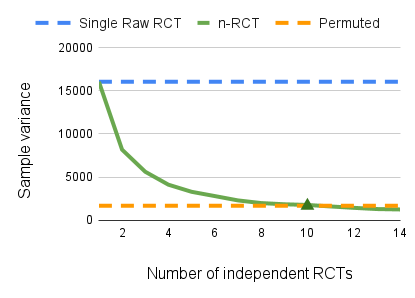}
    \caption{Variance reduces by running and averaging over $n$-independently run trials.}
    \label{fig:variance-n-trials}
\end{figure}

\begin{figure}[H]
    \centering
    \includegraphics[width=0.29\textwidth]{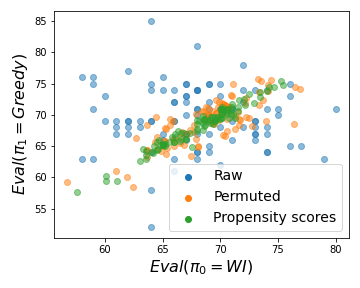}
    \includegraphics[width=0.19\textwidth]{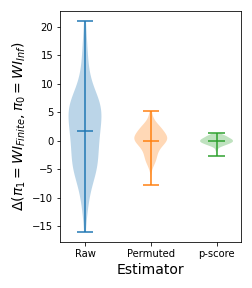}
    \caption{Single shot setup}
    \label{fig:single-shot-low-budget}
\end{figure}

\subsection{Sequential RCTs}
\label{app:experimental-results-sequential}
We run more comparisons using $N=100$ individuals per arm, simulating $500$ instances of trials for $T=10$ timesteps. The $n-$values are listed in Table~\ref{variance-table}.

\begin{figure}[h]
    \centering
    \includegraphics[width=0.32\textwidth]{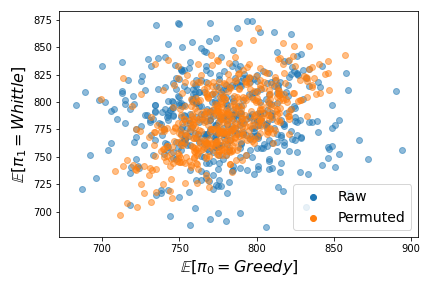}
    \includegraphics[width=0.16\textwidth]{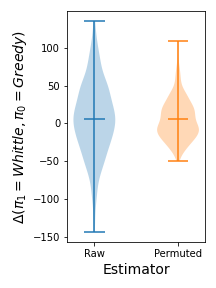}
    \includegraphics[width=0.32\textwidth]{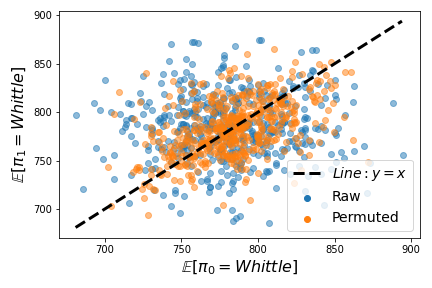}
    \includegraphics[width=0.17\textwidth]{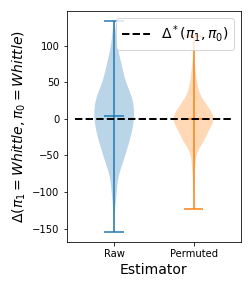}
    \caption{Whittle vs Greedy (left two panels) and Whittle vs whittle (right two panels)}
    \label{fig:sequential-whittle-greedy}
\end{figure}


\paragraph{Setup: Greedy vs Greedy}
For more granular analysis, we consider the state trajectories of individuals participating in the trial. This uses $N=1000$ individuals per arm and simulates $30$ instances of trials for $T=10$ timesteps. We see that orange trajectories (after permutation) is closer to the expected value than blue trajectories (blue)
\begin{figure}[h]
    \centering
    \includegraphics[width=0.5\textwidth]{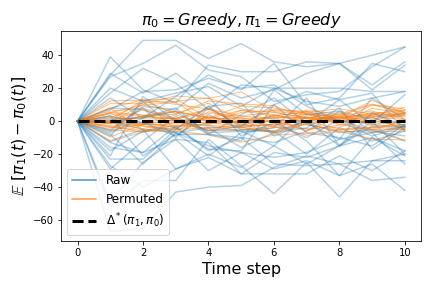}
    \caption{Time trajectories of Greedy v Greedy}
    \label{fig:time-trajectories}
\end{figure}


\end{document}